\pdfoutput=1
\documentclass[11pt,letterpaper]{article}
\usepackage{arxiv}

\usepackage{amsmath,amssymb,amsfonts,amsthm,mathtools}
\usepackage{aliascnt}
\usepackage{bbm}
\usepackage{xcolor}
\usepackage{graphicx}
\usepackage{enumitem}
\usepackage{booktabs}
\usepackage{tabularx}
\usepackage{array}
\usepackage{algorithm}
\usepackage{algpseudocode}
\usepackage{nicefrac}
\usepackage{xspace}
\usepackage[colorlinks=true,allcolors=blue,hypertexnames=false]{hyperref}
\hypersetup{
  pdftitle={Nested Convex-Body Chasing for Online Optimization with Evolving Feasible Sets},
  pdfauthor={Dhruv Sarkar and Aprameyo Chakrabartty}
}
\usepackage[capitalize,noabbrev,nameinlink]{cleveref}
\usepackage[numbers]{natbib}

\newtheoremstyle{ncbevfplain}{6pt}{6pt}{\itshape}{}{\bfseries}{.}{0.5em}{}
\newtheoremstyle{ncbevfdefinition}{6pt}{6pt}{\normalfont}{}{\bfseries}{.}{0.5em}{}

\theoremstyle{ncbevfplain}
\newtheorem{ncbevfthm}{Theorem}[section]
\newaliascnt{ncbevflem}{ncbevfthm}
\newtheorem{ncbevflem}[ncbevflem]{Lemma}
\aliascntresetthe{ncbevflem}
\newaliascnt{ncbevfprop}{ncbevfthm}
\newtheorem{ncbevfprop}[ncbevfprop]{Proposition}
\aliascntresetthe{ncbevfprop}
\newaliascnt{ncbevfcor}{ncbevfthm}

\aliascntresetthe{ncbevfcor}

\theoremstyle{ncbevfdefinition}
\newaliascnt{ncbevfdef}{ncbevfthm}
\newtheorem{ncbevfdef}[ncbevfdef]{Definition}
\aliascntresetthe{ncbevfdef}
\newaliascnt{ncbevfassump}{ncbevfthm}
\newtheorem{ncbevfassump}[ncbevfassump]{Assumption}
\aliascntresetthe{ncbevfassump}
\newaliascnt{ncbevfrem}{ncbevfthm}

\aliascntresetthe{ncbevfrem}

\crefname{ncbevfthm}{theorem}{theorems}
\Crefname{ncbevfthm}{Theorem}{Theorems}
\crefname{ncbevflem}{lemma}{lemmas}
\Crefname{ncbevflem}{Lemma}{Lemmas}
\crefname{ncbevfprop}{proposition}{propositions}
\Crefname{ncbevfprop}{Proposition}{Propositions}
\crefname{ncbevfcor}{corollary}{corollaries}
\Crefname{ncbevfcor}{Corollary}{Corollaries}
\crefname{ncbevfdef}{definition}{definitions}
\Crefname{ncbevfdef}{Definition}{Definitions}
\crefname{ncbevfassump}{assumption}{assumptions}
\Crefname{ncbevfassump}{Assumption}{Assumptions}
\crefname{ncbevfrem}{remark}{remarks}
\Crefname{ncbevfrem}{Remark}{Remarks}

\algrenewcommand\algorithmicrequire{\textbf{Input:}}
\algrenewcommand\algorithmicensure{\textbf{Output:}}

\DeclareMathOperator{\diam}{diam}
\DeclareMathOperator{\dist}{dist}

\DeclareMathOperator*{\argmin}{arg\,min}

\newcommand{\R}{\mathbb R}

\newcommand{\E}{\mathbb E}
\newcommand{\B}{B_2}

\newcommand{\COCO}{\textsf{COCO}\xspace}
\newcommand{\CONES}{\textsf{CONES}\xspace}
\newcommand{\CCV}{\mathsf{CCV}}
\newcommand{\Reg}{\mathsf{Reg}}
\newcommand{\Chase}{\textsc{Chase}}
\newcommand{\M}{\mathsf{M}}
\newcommand{\norm}[1]{\left\lVert #1 \right\rVert}
\newcommand{\ip}[2]{\left\langle #1,#2\right\rangle}

\newcommand{\set}[1]{\left\{#1\right\}}

\allowdisplaybreaks
\renewcommand{\today}{September 5, 2026}

\title{Nested Convex-Body Chasing for Online Optimization with Evolving Feasible Sets}
\author{Dhruv Sarkar\textsuperscript{1} \quad Aprameyo Chakrabartty\textsuperscript{2}\\[3pt]
\normalfont\small\textsuperscript{1} Massachusetts Institute of Technology, Cambridge, MA\\[-1pt]
\normalfont\small\textsuperscript{2} Purdue University, West Lafayette, IN\\[2pt]
\normalfont\small\texttt{dhruv.sarkar@mit.edu}\quad\texttt{chakr190@purdue.edu}}
\date{September 5, 2026}

\begin{document}
\maketitle

\begin{abstract}
We consider online optimization when the feasible set shrinks over time. Two models are studied. In convex optimization with nested evolving feasible sets (\CONES), a fixed objective is known in advance and the learner observes the current feasible set before choosing an action. In constrained online convex optimization (\COCO), the learner acts before seeing the current loss and constraint, and the constraints revealed so far define a nested sequence of feasible sets. If the final feasible set were known from the beginning, \CONES would reduce to one constrained optimization problem, while \COCO would reduce to ordinary online convex optimization on that set. The difficulty is that only the current approximation to the final feasible set is available.

Our algorithms combine nested convex-body chasing with a loss test. The chaser proposes feasible points with small movement. A proposal is accepted when it preserves the required cumulative-loss bound; otherwise the algorithm moves to a constrained minimizer and starts a new phase. For strongly convex \CONES, restricting the chaser to suitable objective sublevel sets gives nonpositive regret at every prefix and $O(\sqrt{\log T})$ movement. We prove a matching $\Omega(\sqrt{\log T})$ lower bound already in two dimensions for every randomized algorithm with terminal expected regret $O(T^\beta)$, for any fixed $\beta<1$. Thus the dependence on the horizon is optimal. If the objective grows at least linearly away from each current constrained minimizer set, Steiner-point tracking gives movement independent of $T$.

For general convex \COCO, we maintain a feasible sequence and play it with one-round delay. Adding a quadratic regularizer yields $O(\sqrt{T})$ regret and $O(\sqrt{T})$ cumulative constraint violation. For strongly convex losses, the corresponding rates are $O(\log T)$ for both regret and cumulative constraint violation. Across these results, using the known Euclidean nested-body chasing guarantee of $O(\sqrt{d\log d})$ replaces the $O(d^{d/2})$ dimension dependence of prior \COCO and \CONES analysis by a polynomial dimension dependence.
\end{abstract}

\newpage
\tableofcontents
\newpage

\section{Introduction}
\label{sec:introduction}

Online convex optimization studies repeated decisions under losses that are revealed only after the decision is made \citep{hazan2022introduction}.  In the standard model, every action is chosen from the same convex set.  Many sequential problems have a second source of uncertainty: the feasible set itself changes as new information arrives.  A safety certificate may exclude an action that was previously allowed, a resource estimate may tighten, or a new constraint may be learned from data.  When each new constraint only removes actions, the feasible sets form a nested sequence.

This paper studies two models with this structure.  In convex optimization with nested evolving feasible sets (\CONES), the objective function is fixed and known.  The learner sees the current feasible set before acting and must choose a feasible point.  Performance is measured by regret relative to the final constrained optimum and by the total distance travelled.  In constrained online convex optimization (\COCO), the learner acts first and then observes the loss and constraint of the round.  The constraints seen so far again define nested feasible sets, but the current set becomes known only after the action has been played.  Performance is measured by regret and cumulative constraint violation.

The final feasible set explains the difficulty in both models.  Suppose it were known before the first round.  In \CONES, the learner could play a minimizer of the fixed objective over that set on every round, giving zero movement and nonpositive regret.  In \COCO, the learner could run a standard online convex optimization method directly on the final feasible set, giving zero violation and the usual regret bound.  The final set is not available.  The learner sees only a sequence of larger approximations that shrink toward it.

\paragraph{The main algorithmic idea.}
A natural response is to move through the sets as they arrive.  Nested convex-body chasing is designed for exactly this geometric task: after each set is revealed, a chaser chooses a point in it and guarantees that its total movement is not much larger than that of the best path that knew all sets in advance.  However, a chaser does not account for the objective or the losses.  A path can move very little and still have poor regret.

Our algorithms add one test to the chaser.  Time is divided into phases.  Within a phase, the chaser proposes a feasible point.  The algorithm uses that point if doing so preserves a cumulative-loss bound.  If the test fails, the algorithm moves to a suitable constrained minimizer and starts a new phase.  Thus the chaser controls movement between resets, while the loss test decides when a reset is necessary.  The same rule appears in both models, but the constrained minimizer and the loss test are chosen differently.

For strongly convex \CONES, this architecture combines two design principles from \citet{krishna2026conesv2}.  Its outer cumulative-loss test and reset to a constrained minimizer follow \textsc{Frugal}, while its restriction to objective sublevel sets follows the Level Set Projection (LSP) method and its Steiner-point variant.  The new ingredient is a stateful, resettable nested-body chaser run at adaptive value scales: tentative chaser states may be discarded, and an endpoint argument charges the entire movement of a phase, including its final reset, to the increase in the constrained optimum.  These features are what yield the sharper horizon dependence and the polynomial dependence on dimension.

For strongly convex \CONES, a phase begins at the current minimizer of the fixed objective.  The chaser is run not on the whole feasible set, but on its intersection with an objective sublevel set.  The sublevel threshold is raised only when the constrained optimum value has more than doubled its increase from the start of the phase.  Strong convexity then turns an objective-value increase $R$ into a distance of order $\sqrt{R/\mu}$.  The thresholds grow geometrically, so the square roots of the thresholds add to the square root of the final increase.  The loss test also implies that reset times more than double.  Combining these two facts gives a $\sqrt{\log T}$ movement bound.

For \COCO, the algorithm keeps a point $z_t$ that is feasible for all constraints through round $t$, and it plays $z_{t-1}$ on round $t$.  This one-round delay is forced by the order of information.  It is also useful: because $z_t$ satisfies the new constraint, the violation on round $t$ is at most the Lipschitz constant times $\|z_t-z_{t-1}\|$.  The same movement term controls the loss caused by playing the delayed point.  It therefore remains to construct a feasible sequence with small movement.  Strong convexity of the cumulative loss provides the needed control in one case.  For general convex losses, we add a fixed quadratic term to the cumulative loss.

\subsection{Results}
\label{sec:intro-results}

We write $\rho_d$ for the competitive ratio of the Euclidean nested-body chaser used by the algorithms.  The known bound is
\[
        \rho_d=O\!\left(\sqrt{d\log(1+d)}\right)
        \qquad\text{\citep{bubeck2020nested}}.
\]
All dependence on the dimension in our upper bounds comes through this quantity.

\paragraph{Strongly convex \CONES.}
Let $v_t$ be the minimum of the fixed objective over the feasible set at time $t$.  For a $G$-Lipschitz, $\mu$-strongly convex objective on a domain of diameter $D$, \Cref{alg:cones-sc} has nonpositive regret at every prefix and movement
\[
        O\!\left(\rho_d
        \sqrt{\frac{(v_T-v_1)(1+\log T)}{\mu}}\right)
        \le
        O\!\left(\rho_d\sqrt{\frac{GD\log(eT)}{\mu}}\right).
\]
The first form is finer: when the constrained optimum changes little, the path is correspondingly shorter.  We prove the complete upper-bound argument in \Cref{sec:cones-strongly-convex}.

The horizon dependence is best possible.  In two dimensions, with a fixed domain and the fixed objective $f(x)=\frac12\|x\|^2$, every randomized algorithm whose terminal expected regret is at most $R_T$ must have expected movement
\[
        \Omega\!\left(\sqrt{\log\frac{T}{R_T+1}}\right)
\]
on some deterministic nested sequence.  In particular, any guarantee $R_T=O(T^\beta)$ with fixed $\beta<1$ forces $\Omega(\sqrt{\log T})$ movement.  The proof appears in full in \Cref{sec:cones-lower-bound}.

\paragraph{Objectives with linear growth away from the minimizer set.}
A different behavior is possible when, on every feasible set, the objective increases at least linearly with distance from the set of minimizers.  In this case, \Cref{alg:sharp} plays the Steiner point of the current minimizer set.  It is optimal on every round and has movement
\[
        dD+\frac{2d}{\alpha}(v_T-v_1)
        \le
        dD+\frac{2dGD}{\alpha},
\]
which does not depend on $T$.

\paragraph{Strongly convex \COCO.}
For $G_f$-Lipschitz, $\mu$-strongly convex losses and $G_g$-Lipschitz constraints, \Cref{alg:coco} satisfies
\[
        \Reg_T\le G_fP_T,
        \qquad
        \CCV_T\le G_gP_T,
\]
where the movement of its feasible sequence obeys
\[
        P_T
        \le
        (1+\rho_d)D+
        \rho_d\left(D+\frac{2G_f\rho_d}{\mu}\right)\log(eT).
\]
Thus regret and violation are logarithmic in $T$.  The proof is given in \Cref{sec:coco-strong}.

\paragraph{General convex \COCO.}
For convex $G_f$-Lipschitz losses, \Cref{alg:coco-convex} adds a quadratic regularizer with parameter $\lambda>0$.  Its feasible-sequence movement satisfies
\[
        P_T\le \rho_dD+\frac{2G_f\rho_d^2T}{\lambda},
\]
and
\[
        \Reg_T\le \frac{\lambda D^2}{2}+G_fP_T,
        \qquad
        \CCV_T\le G_gP_T.
\]
Choosing $\lambda=2G_f\rho_d\sqrt{T}/D$ gives
\[
        \Reg_T=O(G_fD\rho_d\sqrt T),
        \qquad
        \CCV_T=O(G_gD\rho_d\sqrt T).
\]
These rates have no additional logarithmic factor.  A second choice of $\lambda$ removes the dimension factor from the growing $\sqrt T$ term in the violation bound, at the cost of a larger regret constant.

\section{Related work}
\label{sec:related}

\paragraph{Online optimization with adversarial constraints.}
Online convex optimization with long-term constraints has been studied through primal--dual methods, penalties, and surrogate losses; see, for example, \citet{mahdavi2012trading} and \citet{sinha2024optimal}.  In the adversarial \COCO model used here, the action is selected before both the loss and the constraint are revealed.  \citet{vaze2025instance} studied instance-dependent violation bounds and emphasized the nested feasible sets produced by the revealed constraints.

The closest \COCO result is \citet{sarkar2026geometric}.  That paper runs projected gradient descent on the current feasible set and uses the finite-length theory of self-contracted curves to prove $O(\sqrt T)$ regret and violation for convex losses and logarithmic bounds for strongly convex losses.  Our \COCO results have the same dependence on $T$.  The difference lies in the geometric argument and in the dependence on $d$: we use a nested-body chaser and obtain the explicit factor $\rho_d=O(\sqrt{d\log(1+d)})$, while the projection-path argument uses a finite-length constant that can grow superpolynomially with dimension.

\paragraph{Convex-body chasing.}
Convex-body chasing was introduced by \citet{friedman1993chasing}.  Nested requests were shown to be chaseable by \citet{bansal2020nested}; \citet{argue2019chasing} obtained a nearly linear bound in general norms; and \citet{bubeck2020nested} proved the nearly optimal Euclidean bound used in this paper.  A short consequence of competitiveness, proved in \Cref{lem:endpoint}, lets us bound the movement within a phase together with the final move to the next reset point.

\paragraph{Convex optimization with nested feasible sets.}
The \CONES model was introduced by \citet{krishna2026conesv1,krishna2026conesv2}.  Two algorithmic ideas from that work are directly relevant here.  The \textsc{Frugal} algorithm proposes a low-movement action, retains it only when the cumulative-loss test remains below $tv_t$, and otherwise moves to a current constrained minimizer.  The Level Set Projection (LSP) algorithm instead restricts its motion to sets of the form $S_t\cap\{x:f(x)\le L+\varepsilon\}$; the same paper observes that the within-phase projection may be replaced by the Steiner point to improve its dimension dependence.  Our strongly convex \CONES algorithm combines these two principles: the outer acceptance and reset rule is \textsc{Frugal}-type, while the inner requests are LSP-type objective sublevel sets.

This combination alone does not give our bound.  We replace the memoryless projection or Steiner-point selector by a start-aware, resettable nested-body chaser, choose the sublevel scale adaptively so that it grows geometrically within a phase, preview and discard the chaser state when the loss test fails, and use endpoint augmentation to control both within-phase movement and the final reset jump.  These ingredients yield a phase cost of order $\rho_d\sqrt{(v_s-v_r)/\mu}$, which leads to $O(\rho_d\sqrt{\log T})$ movement, removes the smoothness assumption, and replaces the projection-path factor $O(d^{d/2})$ by $\rho_d=O(\sqrt{d\log(1+d)})$.

The lower-bound constructions also share a common geometric backbone.  Both use the quadratic objective, nested tangent halfspaces, and alternating constrained minimizers separated by order $K^{-1/2}$, and both force the learner to approach each successive minimizer before revealing the next cut.  The distinction lies in the regret requirement and the phase accounting.  Version~2 proves
\[
        \Omega\!\left(\sqrt{\frac{\log T}{\log\log T}}\right)
\]
under a regret bound at every prefix.  Our stop-or-continue construction works under only terminal expected regret, applies to randomized algorithms, and keeps the exact accumulated benchmark credit from earlier phases.  Together with the local bound $v_{k+1}-v_k=O(K^{-1})$, this gives a constant-factor phase recurrence and hence $\Omega(\sqrt{\log(T/(R_T+1))})$ movement. The two arXiv versions of the original paper differ in important ways; \Cref{app:versions} gives the detailed comparison and records the exact statements used here.

\paragraph{Access used by the algorithms.}
The algorithms compute exact constrained minimizers and use a deterministic nested-body chaser.  The result for objectives with linear growth also evaluates Steiner points.  These operations define the algorithmic model studied in this paper.  Obtaining the same dimension dependence using only projection, separation, or first-order queries is an open question.

\paragraph{Organization.}
\Cref{sec:setup} defines the two models and the chasing tool.  \Cref{sec:cones} gives the \CONES algorithms, upper bounds, and lower bound, with complete proofs.  \Cref{sec:coco} gives the two \COCO algorithms and their complete analyses.  \Cref{sec:conclusion} concludes.  The two appendices contain the version comparison and the lower-dimensional chasing extension.

\section{Models and a geometric tool}
\label{sec:setup}

Throughout, $\|\cdot\|$ is the Euclidean norm.  For a compact set $K\subseteq\R^d$, write $\diam(K):=\sup_{x,y\in K}\|x-y\|$ and $\dist(x,K):=\inf_{y\in K}\|x-y\|$.  The Euclidean unit ball is denoted by $\B$.

\subsection{Convex optimization with nested evolving feasible sets}
\label{sec:cones-model}

In \CONES, a compact convex set $X\subseteq\R^d$ and a continuous convex objective $f:X\to\R$ are fixed.  On round $t$, the learner first observes a nonempty compact convex set $S_t$ and then chooses $x_t\in S_t$, where $X=S_0\supseteq S_1\supseteq\cdots\supseteq S_T$.  Write $v_t:=\min_{x\in S_t}f(x)$ and $A_t:=\argmin_{x\in S_t}f(x)$ for the current optimum value and minimizer set.  The regret benchmark is the final optimum, $\Reg_T^{\CONES}:=\sum_{t=1}^T f(x_t)-Tv_T$, and the internal movement is $\M_T:=\sum_{t=2}^T\|x_t-x_{t-1}\|$.  If an initial point $x_0\in X$ is fixed before $S_1$ is revealed, we also use $\M_T(x_0):=\|x_1-x_0\|+\M_T$.  When $\diam(X)\le D$, the two notions differ by at most $D$.

\subsection{Constrained online convex optimization}
\label{sec:coco-model}

In \COCO, the learner chooses $x_t\in X$ before seeing the loss and constraint of round $t$.  The adversary then reveals a convex loss $f_t:X\to\R$ and a scalar convex constraint $g_t:X\to\R$.  Set $S_0:=X$ and $S_t:=S_{t-1}\cap\{x\in X:g_t(x)\le0\}$, and assume common feasibility, $S_T\neq\emptyset$.  The regret and cumulative constraint violation are
\[
        \Reg_T
        :=
        \sum_{t=1}^T f_t(x_t)
        -
        \min_{x\in S_T}\sum_{t=1}^T f_t(x),
        \qquad
        \CCV_T:=\sum_{t=1}^T[g_t(x_t)]_+.
\]
The adversary may be adaptive.  The upper bounds below hold for every realized sequence satisfying the stated assumptions.

\begin{ncbevfassump}[Convex \COCO]
\label{ass:coco-convex}
The set $X\subseteq\R^d$ is compact and convex, with diameter at most $D$.  Each loss $f_t$ is convex and $G_f$-Lipschitz on $X$.  Each constraint $g_t$ is convex and $G_g$-Lipschitz on $X$.  Every set $S_t$ is nonempty.
\end{ncbevfassump}

\begin{ncbevfassump}[Strongly convex \COCO]
\label{ass:coco}
\Cref{ass:coco-convex} holds, and every loss $f_t$ is $\mu$-strongly convex on $X$, where $\mu>0$.
\end{ncbevfassump}

The following elementary fact will be used several times.  It remains valid when the minimizer lies on the boundary of the feasible set.

\begin{ncbevflem}[Strong convexity at a constrained minimizer]
\label{lem:constrained-strong-convexity}
Let $K$ be convex, let $F$ be $\gamma$-strongly convex on $K$, and let $u\in\argmin_{x\in K}F(x)$.  Then $F(y)-F(u)\ge \frac{\gamma}{2}\|y-u\|^2$ for every $y\in K$.
\end{ncbevflem}

\begin{proof}
Fix $y\in K$ and $\theta\in(0,1)$.  The point $u_\theta=(1-\theta)u+\theta y$ belongs to $K$.  Optimality of $u$ and strong convexity give
\[
\begin{aligned}
        F(u)
        &\le F(u_\theta)\\
        &\le (1-\theta)F(u)+\theta F(y)
        -\frac{\gamma}{2}\theta(1-\theta)\|y-u\|^2.
\end{aligned}
\]
Subtract $(1-\theta)F(u)$, divide by $\theta$, and let $\theta\downarrow0$.
\end{proof}

Under \Cref{ass:coco}, let $F_t:=\sum_{s=1}^t f_s$ and $u_t:=\argmin_{x\in S_t}F_t(x)$.  Then $F_t$ is $\mu t$-strongly convex, and therefore
\begin{equation}
\label{eq:strong-min-gap}
        F_t(y)-F_t(u_t)
        \ge
        \frac{\mu t}{2}\|y-u_t\|^2,
        \qquad y\in S_t.
\end{equation}

\subsection{Nested convex-body chasing}
\label{sec:chasing}

A nested-body chaser starts at a point $a$ and receives compact convex sets $K_1\supseteq K_2\supseteq\cdots\supseteq K_n$.  After seeing $K_i$, it chooses a point $y_i\in K_i$.  Its movement is compared with the shortest feasible path that knew the entire sequence in advance.

\begin{ncbevfdef}[Resettable nested-body chaser]
\label{def:chaser}
A chaser started at $y_0=a$ is $\rho_d$-competitive if, for every nested request sequence,
\[
        \sum_{i=1}^n\|y_i-y_{i-1}\|
        \le
        \rho_d
        \inf_{\substack{w_i\in K_i\\w_0=a}}
        \sum_{i=1}^n\|w_i-w_{i-1}\|.
\]
It is resettable if a fresh instance of the same rule can be started from any prescribed point.
\end{ncbevfdef}

The algorithms use a deterministic chaser.  Hence its response to the current request is determined by the retained request history.  An algorithm may compute this next response, test it, and retain the updated history only when the response is accepted.  After a rejection, it starts a fresh chaser at the reset point.

\begin{ncbevfthm}[Euclidean nested-body chasing]
\label{thm:known-chaser}
There is a deterministic resettable Euclidean chaser with $\rho_d=O\!\left(\sqrt{d\log(1+d)}\right)$.  It may be applied to arbitrary nonempty compact convex requests, including lower-dimensional sets and singletons.
\end{ncbevfthm}

The full-dimensional result is due to \citet{bubeck2020nested}.  The extension to compact requests with empty interior is proved in \Cref{app:compact-requests}.

We next record the only property of the chaser used in the analyses.  It controls the movement within a phase and the final move to any point that remains feasible throughout the phase.

\begin{ncbevflem}[A phase bound from its endpoint]
\label{lem:endpoint}
Let a $\rho_d$-competitive chaser start from $a$, and let $K_1\supseteq K_2\supseteq\cdots\supseteq K_n$ produce outputs $y_1,\ldots,y_n$.  If $b\in K_n$, then
\[
        \sum_{i=1}^n\|y_i-y_{i-1}\|+\|y_n-b\|
        \le
        \rho_d\|a-b\|,
        \qquad y_0=a.
\]
Moreover, $\|y_j-b\|\le\rho_d\|a-b\|$ for every $j\le n$.
\end{ncbevflem}

\begin{proof}
Append the singleton $\{b\}$ to the request sequence.  This is valid because $b\in K_n\subseteq K_i$ for every $i\le n$.  The chaser is online and deterministic, so its first $n$ outputs do not change.  The offline path can move directly from $a$ to $b$ and then stay there, at cost $\|a-b\|$.  Competitiveness gives the first inequality.  Applying the same argument to the prefix ending at $K_j$ and dropping the nonnegative movement terms gives the second.
\end{proof}

\section{CONES: a fixed objective and shrinking feasible sets}
\label{sec:cones}

We begin with \CONES because the objective is fixed and the current feasible set is known before the action.  This lets us isolate the main use of nested-body chasing before turning to delayed constraints.

\subsection{Strongly convex objectives}
\label{sec:cones-strongly-convex}

Assume that $f$ is $G$-Lipschitz and $\mu$-strongly convex on $X$.  The minimizer over each $S_t$ is unique; write $u_t:=\argmin_{x\in S_t}f(x)$ and $v_t:=f(u_t)$.  The values $v_t$ are nondecreasing.

The chaser cannot be run on $S_t$ alone, because a point with small movement may have a large objective value.  We therefore restrict each request by an objective sublevel set.  The threshold is changed only when the current optimum value has moved far enough.

\paragraph{The sublevel proposal within one phase.}
Suppose a phase begins at round $r$ with point $u_r$.  Let $d_t:=v_t-v_r$.  While $d_t=0$, the proposal is $u_r$.  At the first later round $p$, set $R=d_p$ and start a fresh chaser.  During that sublevel epoch, feed the chaser the requests
\begin{equation}
\label{eq:sublevel-request}
        K_t:=S_t\cap\{x:f(x)\le v_r+2R\}.
\end{equation}
If $d_t>2R$, end the current epoch before round $t$, set $R=d_t$, and start a new chaser from the preceding proposal.  Thus the values of $R$ more than double from one epoch to the next.  This procedure uses only $S_t$ and $v_t$ at time $t$.

The full algorithm keeps the proposal only when it preserves the loss bound $C_t\le tv_t$, where $C_t$ is the cumulative objective value of the actions played so far.

\paragraph{Relation to \textsc{Frugal} and LSP.}
The outer rule above is the cumulative-loss budget mechanism of \textsc{Frugal} \citep{krishna2026conesv2}: a tentative low-movement point is accepted only if $C_{t-1}+f(y_t)\le tv_t$, and a failed test forces a jump to $u_t$.  The inner requests in \eqref{eq:sublevel-request} use the objective-sublevel localization introduced by LSP; its Steiner-point variant already shows that a global selector for the localized sets can improve the dimension dependence.  Our proposal mechanism is nevertheless not a direct substitution of the Steiner point into \textsc{Frugal}.  The scale $R$ is adapted and more than doubles across sublevel epochs, the chaser is stateful and is restarted from the current phase head, its tentative state is committed only when the budget test passes, and \Cref{lem:endpoint} controls the reset jump by appending the phase endpoint only in the analysis.  These features produce the phase bound in \Cref{lem:objective-localized-chasing}, which is the source of the square-root-logarithmic movement guarantee.

\begin{algorithm}[t]
\caption{Budgeted sublevel-set chasing for strongly convex \CONES}
\label{alg:cones-sc}
\begin{algorithmic}[1]
\Require Nested sets $S_t$, objective $f$, resettable chaser $\Chase$
\State Observe $S_1$, compute $u_1\gets\argmin_{x\in S_1}f(x)$, and set $v_1\gets f(u_1)$.
\State Play $x_1\gets u_1$, set $C_1\gets v_1$, and start a phase at $r\gets1$.
\State Initialize the sublevel proposal described above at $u_1$.
\For{$t=2,3,\ldots,T$}
    \State Observe $S_t$, compute $u_t\gets\argmin_{x\in S_t}f(x)$, and set $v_t\gets f(u_t)$.
    \State Compute the tentative sublevel proposal $y_t\in S_t$ for the current phase.
    \If{$C_{t-1}+f(y_t)\le tv_t$}
        \State Play $x_t\gets y_t$ and retain the updated chaser history.
    \Else
        \State Play $x_t\gets u_t$, discard the tentative update, set $r\gets t$, and restart at $u_t$.
    \EndIf
    \State Set $C_t\gets C_{t-1}+f(x_t)$.
\EndFor
\end{algorithmic}
\end{algorithm}

The next lemma gives the two properties of the sublevel proposal that are used in the analysis.  The first bounds its objective value.  The second bounds all movement in a phase, including the final move to the minimizer at the end of the phase.

\begin{ncbevflem}[Adaptive sublevel-set chasing]
\label{lem:objective-localized-chasing}
Let $S_r\supseteq S_{r+1}\supseteq\cdots\supseteq S_s$ be nonempty compact convex sets, and set $a:=u_r$, $b:=u_s$, and $\Delta:=v_s-v_r$.  The sublevel proposal above produces points $y_t\in S_t$ for $t=r+1,\ldots,s$, with $y_r=a$ and $f(y_t)-v_s\le\Delta$ for every such $t$, and
\[
        \sum_{t=r+1}^{s}\|y_t-y_{t-1}\|+\|y_s-b\|
        \le
        12\rho_d\sqrt{\frac{\Delta}{\mu}}.
\]
The point $b$ and the endpoint $s$ are used only in the guarantee, not by the online proposal.
\end{ncbevflem}

\begin{proof}
By \Cref{lem:constrained-strong-convexity}, since $a$ minimizes $f$ over $S_r$,
\begin{equation}
\label{eq:cones-distance-from-start}
        f(x)-f(a)\ge \frac{\mu}{2}\|x-a\|^2
        \qquad\text{for every }x\in S_r.
\end{equation}
As long as $d_t=v_t-v_r$ is zero, the proposal stays at $a$.  This is feasible because the minimizer over $S_r$ is unique and $d_t=0$ implies $u_t=a\in S_t$.

Consider a nontrivial sublevel epoch with scale $R$ and starting point $q$.  In the first epoch, $q=a$.  In every later epoch, the preceding scale is less than $R/2$, so the construction gives $f(q)-v_r<R$.  Hence \eqref{eq:cones-distance-from-start} implies
\[
        \|q-a\|\le \sqrt{\frac{2R}{\mu}}.
\]
Suppose first that this is not the final epoch and let $e$ be its last round.  Since no new epoch began at $e$, $d_e\le2R$.  Thus $u_e$ belongs to the request at every round of this epoch, and
\[
        \|u_e-a\|\le 2\sqrt{\frac{R}{\mu}}.
\]
The offline path can move from $q$ directly to $u_e$ and then stay there.  Its movement is at most
\[
        \|q-u_e\|
        \le
        (2+\sqrt2)\sqrt{\frac{R}{\mu}}.
\]
The chaser therefore moves at most $\rho_d$ times this amount.

For the final epoch, append the singleton $\{b\}$ in the analysis.  Since $\Delta\le2R$ in the final epoch, $b$ belongs to every request in that epoch.  \Cref{lem:endpoint} gives the same bound for the chaser movement plus $\|y_s-b\|$.

Let $R_1,\ldots,R_E$ be the scales used through round $s$.  They more than double, and the final scale is at most $\Delta$.  Consequently,
\[
        \sum_{e=1}^E\sqrt{R_e}
        \le
        (2+\sqrt2)\sqrt{\Delta}.
\]
Summing the epoch bounds gives
\[
        \sum_{t=r+1}^{s}\|y_t-y_{t-1}\|+\|y_s-b\|
        \le
        12\rho_d\sqrt{\frac{\Delta}{\mu}}.
\]

It remains to bound the objective.  A nonconstant proposal in an epoch of scale $R$ belongs to \eqref{eq:sublevel-request}, so $f(y_t)\le v_r+2R$.  The scale was set to $d_p$ at some round $p\le s$, hence $R\le\Delta$.  Therefore
\[
        f(y_t)-v_s
        \le
        v_r+2R-(v_r+\Delta)
        \le \Delta.
\]
The same inequality is immediate when $y_t=a$.
\end{proof}

\begin{ncbevfthm}[Strongly convex \CONES upper bound]
\label{thm:cones-sc}
Suppose $f$ is $G$-Lipschitz and $\mu$-strongly convex on a compact convex set $X$ of diameter at most $D$.  \Cref{alg:cones-sc} has nonpositive regret at every prefix, namely $\sum_{\tau=1}^{t}f(x_\tau)-tv_t\le0$ for every $t\le T$, and
\[
        \M_T
        \le
        C\rho_d\sqrt{\frac{(v_T-v_1)(1+\log T)}{\mu}}
        \le
        C\rho_d\sqrt{\frac{GD\log(eT)}{\mu}},
\]
where $C>0$ is universal.  If $x_0\in X$ is fixed before $S_1$ is revealed, add at most $D$.  With \Cref{thm:known-chaser}, this becomes $\M_T=O\!\left(\sqrt{d\log(1+d)}\sqrt{GD\log(eT)/\mu}\right)$.
\end{ncbevfthm}

\begin{proof}
Let $C_t=\sum_{\tau=1}^t f(x_\tau)$.  We first show that $C_t\le tv_t$ for every $t$.  This holds at $t=1$.  If the tentative proposal is accepted at time $t$, the test gives the claim directly.  If it is rejected, the algorithm plays $u_t$, and
\[
        C_t=C_{t-1}+v_t
        \le (t-1)v_{t-1}+v_t
        \le tv_t.
\]
Thus the bound holds by induction.  Since $v_t\le v_T$, it gives nonpositive regret at every prefix.

Let $1=r_0<r_1<\cdots<r_m\le T$ be the reset times.  Consider a completed phase that starts at $r$ and resets at $s$, and write $\Delta=v_s-v_r$.  The tentative proposal $y_s$ is rejected, so
\[
        C_{s-1}+f(y_s)>sv_s.
\]
All earlier proposals in the phase were accepted.  Since $C_r\le rv_r$,
\[
        \sum_{t=r+1}^{s}f(y_t)>sv_s-rv_r.
\]
After subtracting $(s-r)v_s$,
\[
        \sum_{t=r+1}^{s}\bigl(f(y_t)-v_s\bigr)>r\Delta.
\]
Each summand is at most $\Delta$ by \Cref{lem:objective-localized-chasing}.  The strict inequality rules out $\Delta=0$, and therefore $(s-r)\Delta>r\Delta$.  Hence $s>2r$.  The reset times more than double, so there are at most $1+\log_2T$ completed phases.

For a completed phase from $r_k$ to $r_{k+1}$, let $\Delta_k=v_{r_{k+1}}-v_{r_k}$.  For the final unfinished phase, let $\Delta_m=v_T-v_{r_m}$.  These increments satisfy
\[
        \sum_{k=0}^{m}\Delta_k=v_T-v_1.
\]
The phase bound in \Cref{lem:objective-localized-chasing} includes the move to the next reset point.  It therefore gives
\[
        \M_T
        \le
        12\rho_d\sum_{k=0}^{m}\sqrt{\frac{\Delta_k}{\mu}}.
\]
For the last phase, append $u_T$ in the analysis and then drop the final nonnegative distance.  Cauchy--Schwarz now yields
\[
\begin{aligned}
        \sum_{k=0}^{m}\sqrt{\Delta_k}
        &\le
        \sqrt{(m+1)\sum_{k=0}^{m}\Delta_k}\\
        &\le
        \sqrt{(1+\log_2T)(v_T-v_1)}.
\end{aligned}
\]
This proves the first movement bound.  Finally,
\[
        v_T-v_1=f(u_T)-f(u_1)
        \le G\|u_T-u_1\|
        \le GD.
\]
\end{proof}

\subsection{A matching lower bound}
\label{sec:cones-lower-bound}

We next show that the dependence on $T$ in \Cref{thm:cones-sc} cannot be improved.  The result applies to randomized algorithms, while the hard sequence is deterministic once the algorithm has been fixed.  The first lemma reduces the proof to deterministic policies.

\begin{ncbevflem}[Replacing a randomized policy by its mean]
\label{lem:cones-derandomization}
Let $\mathcal A$ be a randomized feasible \CONES algorithm.  For a deterministic request prefix $S_{1:t}$, let $X_t(S_{1:t};\omega)$ be its action under random seed $\omega$, and set $\bar x_t(S_{1:t}):=\E_\omega[X_t(S_{1:t};\omega)]$.  Then $\bar x_t(S_{1:t})\in S_t$, and on every fixed nested sequence and for every convex objective $f$,
\[
        \sum_{t=1}^T f(\bar x_t)
        \le
        \E\!\left[\sum_{t=1}^T f(X_t)\right],
        \qquad
        \sum_{t=2}^T\|\bar x_t-\bar x_{t-1}\|
        \le
        \E\!\left[\sum_{t=2}^T\|X_t-X_{t-1}\|\right].
\]
\end{ncbevflem}

\begin{proof}
For every fixed prefix, $X_t\in S_t$ almost surely.  Convexity of $S_t$ gives $\bar x_t\in S_t$.  Jensen's inequality gives $f(\bar x_t)\le\E[f(X_t)]$, and convexity of the norm gives
\[
        \|\bar x_t-\bar x_{t-1}\|
        =\|\E[X_t-X_{t-1}]\|
        \le\E\|X_t-X_{t-1}\|.
\]
Summing proves the claims.
\end{proof}

\begin{ncbevfthm}[Terminal-regret lower bound for strongly convex \CONES]
\label{thm:terminal-strongly-convex-cones-lower}
There is a universal constant $c>0$ with the following property.  Fix $T$ and $R_T\ge0$, and suppose a randomized feasible \CONES algorithm satisfies $\E[\sum_{t=1}^{T}f(X_t)-Tv_T]\le R_T$ for every nested sequence of nonempty compact convex feasible sets of length $T$.  If $T/(R_T+1)$ is sufficiently large, then there is a deterministic nested sequence in $\R^2$, constructed for that algorithm but independent of its realized random seed, such that
\[
        \E[\M_T]
        \ge
        c\sqrt{\log\!\left(\frac{T}{R_T+1}\right)}.
\]
The construction uses $X=[-1,1]\times[1,2]$ and $f(x)=\frac12\|x\|^2$.  Thus $\diam(X)=\sqrt5$, $f$ is $1$-strongly convex, and $f$ is $\sqrt5$-Lipschitz on $X$, independently of $T$ and $R_T$.  In particular, if $R_T\le C_{\rm reg}T^\beta$ for fixed $C_{\rm reg}>0$ and $\beta<1$, then $\E[\M_T]=\Omega(\sqrt{\log T})$.
\end{ncbevfthm}

\paragraph{Proof idea and relation to the earlier \CONES lower bound.}
The geometric core is closely related to the strongly convex construction of \citet{krishna2026conesv2}.  Both proofs use $f(x)=\frac12\|x\|^2$, nested supporting halfspaces, and $K$ constrained minimizers that alternate between two nearby horizontal positions.  Consecutive minimizers are separated by order $K^{-1/2}$, and the current set is held fixed until the learner approaches its minimizer.  Visiting all $K$ target neighborhoods therefore costs order $K\cdot K^{-1/2}=\sqrt K$ movement.

The proof below differs in the way those visits are forced and counted.  If the learner misses the current target by its deadline, the adversary stops shrinking the sequence, so the current optimum becomes the terminal benchmark and the terminal-regret guarantee is violated.  The deadline retains the exact benchmark credit accumulated in all previous phases.  Since the increase between consecutive optimum values is only $O(K^{-1})$, this credit satisfies a constant-factor recurrence, allowing $K$ to be proportional to $\log(T/(R_T+1))$.  In the earlier prefix-regret proof, a coarser global credit estimate yields a phase-growth factor of order $K$ and hence $K=\Theta(\log T/\log\log T)$.  The improvement comes from the stop-or-continue argument and the sharper local accounting.  Finally, \Cref{lem:cones-derandomization} extends the construction from deterministic to randomized learners.

\begin{proof}
Let \(\mathcal A\) be a randomized algorithm satisfying the theorem's terminal expected-regret guarantee, and form its deterministic mean policy from \Cref{lem:cones-derandomization}.  Apply the deterministic construction below to that mean policy.  This produces a deterministic nested sequence that may depend on the distributional rule of \(\mathcal A\), but not on its realized random seed.  On the resulting fixed sequence, the mean policy has terminal regret at most \(R_T\), while its movement is no larger than the expected movement of \(\mathcal A\).  It is therefore enough to establish the claimed movement lower bound for a deterministic feasible policy.

\paragraph{Choice of parameters.}
Fix such a deterministic policy.  Set
\[
        Z_T:=\frac{T}{R_T+1}
\]
and assume that \(Z_T\) is sufficiently large.  Let
\[
        K:=\left\lfloor \kappa\log Z_T\right\rfloor,
\]
where \(\kappa>0\) is a sufficiently small universal constant to be chosen later.  We may assume \(K\ge2\).  Define
\[
        h:=\frac{1}{8\sqrt K},
        \qquad
        \delta:=4h^2=\frac{1}{16K},
        \qquad
        \varepsilon:=\frac h4,
        \qquad
        \eta:=\frac{\varepsilon^2}{2}=\frac{1}{2048K}.
\]
For \(k=1,\ldots,K\), set
\[
        \sigma_k:=(-1)^k,
        \qquad
        u_k:=(\sigma_k h,1+k\delta).
\]
Since \(K\delta=1/16\), all these points lie in \(X=[-1,1]\times[1,2]\).

\paragraph{Nested sets and their minimizers.}
For each \(k\), define the halfspace
\[
        H_k:=\set{x\in\R^2:\ip{u_k}{x-u_k}\ge0}
\]
and the nested set
\[
        \mathcal S_k:=X\cap\bigcap_{j=1}^{k}H_j.
\]
We first verify that \(u_k\in\mathcal S_k\).  For \(j<k\),
\[
        \ip{u_j}{u_k-u_j}
        =h^2(\sigma_j\sigma_k-1)+(1+j\delta)(k-j)\delta.
\]
The first term is either \(0\) or \(-2h^2\), while the second is at least \(\delta=4h^2\).  Hence the inner product is at least \(2h^2>0\).  Also \(u_k\in H_k\) with equality, so \(u_k\in\mathcal S_k\).

For any \(x\in\mathcal S_k\), membership in \(H_k\) gives \(\ip{u_k}{x-u_k}\ge0\).  Since \(f(x)=\frac12\norm{x}^2\),
\[
        f(x)-f(u_k)
        =
        \ip{u_k}{x-u_k}+\frac12\norm{x-u_k}^2
        \ge
        \frac12\norm{x-u_k}^2.
\]
Thus \(u_k\) is the unique minimizer of \(f\) over \(\mathcal S_k\).  Let \(v_k=f(u_k)\).  If \(\norm{x-u_k}>\varepsilon\), then \(f(x)-v_k>\eta\).

The increase in constrained optimum satisfies, for \(k<K\),
\[
        \Delta_{k+1}:=v_{k+1}-v_k
        =
        \delta\left(1+\left(k+\frac12\right)\delta\right).
\]
Since \(K\delta=1/16\),
\[
        \delta\le \Delta_{k+1}<\frac{17}{16}\delta<5h^2.
\]
Using \(\eta=h^2/32\), this gives \(\Delta_{k+1}\le160\eta\).

\paragraph{How the sets are revealed.}
The adversary reveals the sets \(\mathcal S_1,\ldots,\mathcal S_K\) in phases.  Let \(s_k\) be the first round of phase \(k\), and let \(n_i\) be the realized length of phase \(i\).  Thus
\[
        s_1=1,
        \qquad
        s_k=1+\sum_{i=1}^{k-1} n_i.
\]
At the beginning of phase \(k\), define the accumulated benchmark credit
\[
        P_k:=\sum_{i=1}^{k-1}n_i(v_k-v_i).
\]
This is the largest amount by which losses incurred in earlier phases can lie below the new benchmark value \(v_k\), using only the fact that an action played while the feasible set was \(\mathcal S_i\) has loss at least \(v_i\).  Set
\[
        L_k:=\left\lfloor\frac{P_k+R_T}{\eta}\right\rfloor+1.
\]
During phase \(k\), the adversary repeatedly presents \(\mathcal S_k\).  If the learner enters \(B(u_k,\varepsilon)\) during the first \(L_k\) rounds of the phase, the next phase begins on the following round.  If the learner does not enter this ball during those \(L_k\) rounds, the adversary keeps \(\mathcal S_k\) fixed for all remaining rounds.

\paragraph{Why every phase fits within the horizon.}
We next verify that every deadline used by the construction occurs before time $T$.  Define
\[
        N_k:=\sum_{i=1}^{k}n_i,
        \qquad
        Q_k:=\frac{P_k}{\eta},
        \qquad
        B:=\frac{R_T}{\eta},
\]
with \(N_0=0\), and put
\[
        W_k:=N_{k-1}+Q_k+B+1.
\]
We prove by induction over \(k\) that phase \(k\) begins, that its deadline occurs by time \(T\), and that \(n_k\le L_k\).  The estimates below will also give the final bound on the total number of rounds used.

Assume that phases \(1,\ldots,k-1\) have completed by their deadlines and that
\[
        W_k\le 324^{k-1}(B+1).
\]
This is true for \(k=1\), since \(P_1=N_0=0\) and hence \(W_1=B+1\).  Because
\[
        L_k\le Q_k+B+1\le W_k
        \qquad\text{and}\qquad
        N_{k-1}\le W_k,
\]
the phase-\(k\) deadline occurs no later than
\[
        N_{k-1}+L_k\le 2W_k\le 2\cdot 324^{k-1}(B+1).
\]
Since \(\eta^{-1}=2048K\),
\[
        B+1\le 2049K(R_T+1).
\]
Choose \(\kappa<1/(4\log 324)\).  For every \(k\le K\), the choice
\(K\le \kappa\log Z_T\) gives
\[
        324^{k-1}\le 324^K\le Z_T^{1/4}.
\]
Therefore
\[
        2\cdot324^{k-1}(B+1)
        \le
        C_0K(R_T+1)Z_T^{1/4}
        =
        C_0KTZ_T^{-3/4}.
\]
For sufficiently large \(Z_T\), the final quantity is at most \(T\).  Thus the phase-\(k\) deadline lies within the horizon.

If the learner failed to enter \(B(u_k,\varepsilon)\) by this deadline, the adversary would freeze the feasible set at \(\mathcal S_k\).  The terminal optimum value would then be \(v_k\).  Earlier phases can contribute at worst \(-P_k\) to terminal regret relative to \(v_k\), while the first \(L_k\) rounds of phase \(k\) contribute strictly more than \(L_k\eta\) by the excess-loss bound above.  All later rounds contribute nonnegatively, because the learner remains feasible in \(\mathcal S_k\).  Hence
\[
        \Reg_T^{\CONES}
        >
        -P_k+L_k\eta
        >
        R_T,
\]
contradicting the assumed terminal-regret guarantee.  Thus the learner must enter \(B(u_k,\varepsilon)\) before the deadline, and \(n_k\le L_k\).

It remains to close the induction.  Since \(n_k\le Q_k+B+1\le W_k\), we have
\[
        N_k=N_{k-1}+n_k\le 2W_k.
\]
Furthermore, using \(P_{k+1}=P_k+\Delta_{k+1}N_k\) and \(\Delta_{k+1}\le160\eta\),
\[
        Q_{k+1}\le Q_k+160N_k\le W_k+320W_k=321W_k.
\]
Since \(B+1\le W_k\),
\[
        W_{k+1}
        =N_k+Q_{k+1}+B+1
        \le
        2W_k+321W_k+W_k
        =324W_k.
\]
This proves the induction.

In particular, after \(K\) phases,
\[
        N_K\le 2W_K\le 2\cdot324^{K-1}(B+1)\le T.
\]
Hence all \(K\) phases occur within the horizon.  After the learner enters \(B(u_K,\varepsilon)\), the adversary keeps \(\mathcal S_K\) fixed for any remaining rounds.

\paragraph{Movement forced by the construction.}
Let \(\tau_k\) be the first round in phase \(k\) when the learner enters \(B(u_k,\varepsilon)\).  Then \(\norm{x_{\tau_k}-u_k}\le\varepsilon\).  Consecutive minimizers satisfy
\[
        \norm{u_{k+1}-u_k}=\sqrt{4h^2+\delta^2}\ge2h.
\]
Thus
\[
        \norm{x_{\tau_{k+1}}-x_{\tau_k}}
        \ge
        2h-2\varepsilon
        =
        \frac32 h.
\]
The times \(\tau_1<\cdots<\tau_K\) are ordered, so the corresponding trajectory segments are disjoint.  Therefore
\[
        \M_T
        \ge
        \sum_{k=1}^{K-1}\norm{x_{\tau_{k+1}}-x_{\tau_k}}
        \ge
        \frac32(K-1)h
        \ge
        \frac{3}{32}\sqrt K
\]
for \(K\ge2\).  Since \(K=\Theta(\log Z_T)\), this gives
\[
        \M_T=\Omega\left(\sqrt{\log\left(\frac{T}{R_T+1}\right)}\right).
\]
This proves the lower bound for the deterministic mean policy.  The mean-policy lemma then gives the stated expected-movement bound for the original randomized algorithm.
\end{proof}

\subsection{Objectives with linear growth away from the minimizer set}
\label{sec:cones-sharp}

The last \CONES result uses a stronger property than strong convexity.  It requires the objective gap to control distance from the whole constrained minimizer set.

\begin{ncbevfassump}[Uniform linear growth]
\label{ass:sharp}
There is $\alpha>0$ such that $f(x)-v_t\ge \alpha\dist(x,A_t)$ for every $t$ and every $x\in S_t$, where $A_t:=\argmin_{y\in S_t}f(y)$.
\end{ncbevfassump}

For a nonempty compact convex set $K$, let $s(K)$ denote its Steiner point and write $h_K(\theta):=\sup_{x\in K}\langle\theta,x\rangle$ for its support function.  If $\sigma$ is the rotation-invariant probability measure on the unit sphere, then
\begin{equation}
\label{eq:steiner-definition}
        s(K)=d\int_{\mathbb S^{d-1}}h_K(\theta)\theta\,d\sigma(\theta).
\end{equation}
The Steiner point belongs to $K$, and
\begin{equation}
\label{eq:steiner-lip}
        \|s(K)-s(L)\|
        \le
        d\int_{\mathbb S^{d-1}}|h_K(\theta)-h_L(\theta)|\,d\sigma(\theta).
\end{equation}
These standard facts may be found in \citet{schneider2014convex,bubeck2020nested}.

\begin{algorithm}[t]
\caption{Steiner-point tracking for \CONES}
\label{alg:sharp}
\begin{algorithmic}[1]
\For{$t=1,2,\ldots,T$}
    \State Observe $S_t$, form $A_t=\argmin_{x\in S_t}f(x)$, and play $x_t=s(A_t)$.
\EndFor
\end{algorithmic}
\end{algorithm}

\begin{ncbevfthm}[Horizon-independent movement]
\label{thm:cones-sharp}
Suppose $f$ is $G$-Lipschitz on $X$ and \Cref{ass:sharp} holds.  Then \Cref{alg:sharp} satisfies $\Reg_T^{\CONES}\le0$ and
\[
        \M_T
        \le
        dD+\frac{2d}{\alpha}(v_T-v_1)
        \le
        dD+\frac{2dGD}{\alpha}.
\]
If $x_0\in X$ is fixed before $S_1$ is revealed, add at most $D$.
\end{ncbevfthm}

\begin{proof}
Since $s(A_t)\in A_t$, the action has objective value $v_t$.  The values $v_t$ are nondecreasing, so
\[
        \Reg_T^{\CONES}
        =
        \sum_{t=1}^T(v_t-v_T)
        \le0.
\]

Fix $t\ge2$ and $a\in A_t$.  Because $a\in S_t\subseteq S_{t-1}$, \Cref{ass:sharp} at time $t-1$ gives
\[
        \alpha\dist(a,A_{t-1})
        \le
        f(a)-v_{t-1}
        =v_t-v_{t-1}.
\]
Hence
\begin{equation}
\label{eq:At-inclusion}
        A_t\subseteq A_{t-1}+r_t\B,
        \qquad
        r_t:=\frac{v_t-v_{t-1}}{\alpha}.
\end{equation}
For every unit vector $\theta$, this implies
\begin{equation}
\label{eq:positive-support}
        h_{A_t}(\theta)-h_{A_{t-1}}(\theta)\le r_t.
\end{equation}

For any real sequence $(a_t)$,
\[
        \sum_{t=2}^T|a_t-a_{t-1}|
        \le
        |a_T-a_1|+2\sum_{t=2}^T(a_t-a_{t-1})_+.
\]
Apply this with $a_t=h_{A_t}(\theta)$.  Since every $A_t$ lies in $X$, $|a_T-a_1|\le D$.  By \eqref{eq:positive-support}, the positive increment is at most $r_t$.  Therefore
\[
        \sum_{t=2}^T|h_{A_t}(\theta)-h_{A_{t-1}}(\theta)|
        \le
        D+2\sum_{t=2}^T r_t.
\]
Use \eqref{eq:steiner-lip}, integrate over $\theta$, and exchange the finite sum and the integral.  This gives
\[
        \M_T
        \le
        dD+2d\sum_{t=2}^T r_t
        =
        dD+\frac{2d}{\alpha}(v_T-v_1).
\]
Finally, $v_T-v_1\le GD$ by Lipschitzness.
\end{proof}

\section{COCO: learning the feasible set with one-round delay}
\label{sec:coco}

We now turn to \COCO.  The learner does not know the newest feasible set when it chooses the action of the round.  Both algorithms below handle this delay in the same way.  After observing round $t$, they compute a point $z_t\in S_t$ and play the previously computed point $x_t=z_{t-1}$ on the next round.  Thus $z_t$ is feasible for every constraint seen through round $t$, while the played point is one round behind.

This delay reduces both regret and constraint violation to the movement $P_T:=\sum_{t=1}^T\|z_t-z_{t-1}\|$ of the feasible sequence.  The following lemma explains why $P_T$ is the main quantity to control.

\begin{ncbevflem}[The cost of the one-round delay]
\label{lem:coco-delay}
Suppose $z_t\in S_t$ for every $t$, the learner plays $x_t=z_{t-1}$, and $A_T:=\sum_{t=1}^T f_t(z_t)$.  Then, for every $x^\star\in S_T$,
\[
        \sum_{t=1}^T f_t(x_t)-\sum_{t=1}^T f_t(x^\star)
        \le
        A_T-\sum_{t=1}^T f_t(x^\star)+G_fP_T,
\]
and $\CCV_T\le G_gP_T$.
\end{ncbevflem}

\begin{proof}
Since $x_t=z_{t-1}$,
\[
\begin{aligned}
        \sum_{t=1}^T f_t(x_t)-\sum_{t=1}^T f_t(x^\star)
        &=
        A_T-\sum_{t=1}^T f_t(x^\star)
        +\sum_{t=1}^T\bigl(f_t(z_{t-1})-f_t(z_t)\bigr)\\
        &\le
        A_T-\sum_{t=1}^T f_t(x^\star)+G_fP_T.
\end{aligned}
\]
Also, $z_t\in S_t$, so $g_t(z_t)\le0$.  Therefore
\[
        [g_t(z_{t-1})]_+
        \le
        [g_t(z_{t-1})-g_t(z_t)]_+
        \le
        G_g\|z_t-z_{t-1}\|.
\]
Summing over $t$ proves the violation bound.
\end{proof}

The remaining task is to construct $z_t$ so that its cumulative loss is controlled and $P_T$ is small.  We first use the curvature of strongly convex losses.  We then add a quadratic term when the losses are only convex.

\subsection{Strongly convex losses}
\label{sec:coco-strong}

Assume \Cref{ass:coco}.  Let $F_t(x):=\sum_{s=1}^t f_s(x)$, $u_t:=\argmin_{x\in S_t}F_t(x)$, and $v_t:=F_t(u_t)$.  The point $u_t$ is the best fixed action for the losses seen through time $t$, subject to all constraints seen through time $t$.

The algorithm keeps $A_t:=\sum_{s=1}^t f_s(z_s)$, with $A_0:=0$.
After round $t$, the chaser gives a tentative point in $S_t$.  The point is accepted when adding its current loss keeps $A_t$ below $v_t$.  Otherwise the algorithm sets $z_t=u_t$ and starts a new chaser there.

\begin{algorithm}[t]
\caption{One-round-delayed chasing for strongly convex \COCO}
\label{alg:coco}
\begin{algorithmic}[1]
\Require Initial point $z_0\in X$, resettable chaser $\Chase$
\State Set $S_0\gets X$, $F_0\equiv0$, and $A_0\gets0$.
\State Play $x_1\gets z_0$ and observe $f_1,g_1$.
\State Set $S_1\gets S_0\cap\{x\in X:g_1(x)\le0\}$ and $F_1\gets f_1$.
\State Compute $u_1\gets\argmin_{x\in S_1}F_1(x)$ and $v_1\gets F_1(u_1)$.
\State Set $z_1\gets u_1$, $A_1\gets v_1$, and start $\Chase$ at $z_1$.
\For{$t=2,3,\ldots,T$}
    \State Play $x_t\gets z_{t-1}$ and observe $f_t,g_t$.
    \State Set $S_t\gets S_{t-1}\cap\{x\in X:g_t(x)\le0\}$ and $F_t\gets F_{t-1}+f_t$.
    \State Compute $u_t\gets\argmin_{x\in S_t}F_t(x)$ and $v_t\gets F_t(u_t)$.
    \State Compute the tentative chaser response $y_t\in S_t$.
    \If{$A_{t-1}+f_t(y_t)\le v_t$}
        \State Set $z_t\gets y_t$ and retain the updated chaser history.
    \Else
        \State Set $z_t\gets u_t$, discard the tentative response, and restart $\Chase$ at $u_t$.
    \EndIf
    \State Set $A_t\gets A_{t-1}+f_t(z_t)$.
\EndFor
\end{algorithmic}
\end{algorithm}

We prove the result in three steps.  First, the loss test maintains $A_t\le v_t$.  Second, this inequality and \Cref{lem:coco-delay} reduce regret and violation to $P_T$.  Third, strong convexity bounds the total distance between reset points.

\begin{ncbevflem}[Loss bound for the feasible sequence]
\label{lem:coco-invariant}
For every $t\ge1$, \Cref{alg:coco} satisfies $A_t\le v_t$.  Consequently, $B_t:=v_t-A_t$ is nonnegative.
\end{ncbevflem}

\begin{proof}
The initialization gives $A_1=f_1(u_1)=v_1$.  Suppose $A_{t-1}\le v_{t-1}$.  If the tentative point is accepted, the test gives $A_t\le v_t$.  If it is rejected, then $z_t=u_t$.  Since $u_t\in S_t\subseteq S_{t-1}$,
\[
        v_{t-1}
        =\min_{x\in S_{t-1}}F_{t-1}(x)
        \le F_{t-1}(u_t).
\]
It follows that
\[
\begin{aligned}
        A_t
        &=A_{t-1}+f_t(u_t)\\
        &\le v_{t-1}+f_t(u_t)\\
        &\le F_{t-1}(u_t)+f_t(u_t)
        =F_t(u_t)=v_t.
\end{aligned}
\]
The claim follows by induction.
\end{proof}

Applying \Cref{lem:coco-delay} with $x^\star=u_T$ and using $A_T-v_T=-B_T$ gives
\begin{equation}
\label{eq:coco-regret-movement}
        \Reg_T\le G_fP_T-B_T\le G_fP_T,
        \qquad
        \CCV_T\le G_gP_T.
\end{equation}
It remains to bound $P_T$.

Let $1=r_0<r_1<\cdots<r_m\le T$ be the reset times: $r_0=1$, and $r_k$ for $k\ge1$ is a round on which the tentative chaser response is rejected.  Set $a_k:=z_{r_k}=u_{r_k}$ and $\delta_k:=\|a_{k+1}-a_k\|$ for $k=0,\ldots,m-1$.

\begin{ncbevflem}[What a reset implies]
\label{lem:reset-ineq}
For every completed phase $k=0,\ldots,m-1$,
\[
        (r_{k+1}-r_k)G_f\rho_d\delta_k
        >
        \frac{\mu r_k}{2}\delta_k^2+B_{r_k}.
\]
\end{ncbevflem}

\begin{proof}
Fix a completed phase.  Write $r=r_k$, $s=r_{k+1}$, $a=a_k$, $b=a_{k+1}=u_s$, and $\delta=\|a-b\|$.  Let $y_t$ be the tentative chaser response on $S_t$ for $t=r+1,\ldots,s$.  The responses through $s-1$ are accepted, while $y_s$ is rejected.  Hence
\[
        A_{s-1}+f_s(y_s)>v_s
\]
and
\[
        A_{s-1}=A_r+\sum_{t=r+1}^{s-1}f_t(y_t).
\]
Since $v_s=F_s(b)=F_r(b)+\sum_{t=r+1}^s f_t(b)$, subtraction gives
\begin{equation}
\label{eq:coco-reset-gap}
        \sum_{t=r+1}^s\bigl(f_t(y_t)-f_t(b)\bigr)
        >F_r(b)-A_r.
\end{equation}
Now $a=u_r$, so $F_r(a)=v_r$ and
\[
        F_r(b)-A_r
        =F_r(b)-F_r(a)+B_r.
\]
The point $b$ lies in $S_s\subseteq S_r$.  Applying \eqref{eq:strong-min-gap} at time $r$ yields
\[
        F_r(b)-F_r(a)
        \ge \frac{\mu r}{2}\|b-a\|^2
        =\frac{\mu r}{2}\delta^2.
\]

For the other side of \eqref{eq:coco-reset-gap}, append the point $b$ to the phase in the analysis.  By \Cref{lem:endpoint},
\[
        \|y_t-b\|\le\rho_d\delta
        \qquad (t=r+1,\ldots,s).
\]
Lipschitzness then gives
\[
        f_t(y_t)-f_t(b)\le G_f\rho_d\delta.
\]
Summing these $s-r$ inequalities and comparing with \eqref{eq:coco-reset-gap} proves the claim.
\end{proof}

\begin{ncbevflem}[Total distance between reset points]
\label{lem:delta-sum}
The reset points satisfy
\[
        \sum_{k=0}^{m-1}\delta_k
        \le
        \left(D+\frac{2G_f\rho_d}{\mu}\right)\log(eT).
\]
\end{ncbevflem}

\begin{proof}
Terms with $\delta_k=0$ can be ignored.  For every other term, drop the nonnegative $B_{r_k}$ from \Cref{lem:reset-ineq} and divide by $G_f\rho_d\delta_kr_k$.  We obtain
\[
        \frac{r_{k+1}}{r_k}
        >
        1+\frac{\mu\delta_k}{2G_f\rho_d}.
\]
Multiplying over the completed phases and using $r_0=1$ and $r_m\le T$ gives
\[
        \sum_{k=0}^{m-1}
        \log\left(1+\frac{\mu\delta_k}{2G_f\rho_d}\right)
        \le\log T.
\]
Each $\delta_k$ is at most $D$.  For $0\le x\le M$, the inequality $\log(1+x)\ge x/(1+M)$ holds.  Use it with
\[
        x=\frac{\mu\delta_k}{2G_f\rho_d},
        \qquad
        M=\frac{\mu D}{2G_f\rho_d}.
\]
This gives
\[
        \frac{\mu}{2G_f\rho_d+\mu D}
        \sum_{k=0}^{m-1}\delta_k
        \le\log(eT),
\]
which is the stated bound.
\end{proof}

\begin{ncbevfthm}[Strongly convex \COCO]
\label{thm:coco-main}
Under \Cref{ass:coco}, let \Cref{alg:coco} use a resettable $\rho_d$-competitive chaser.  Then
\[
        P_T
        \le
        (1+\rho_d)D+
        \rho_d\left(D+\frac{2G_f\rho_d}{\mu}\right)\log(eT),
\]
Moreover, $\Reg_T\le G_fP_T-B_T\le G_fP_T$ and $\CCV_T\le G_gP_T$.  With the Euclidean chaser from \Cref{thm:known-chaser}, both quantities are logarithmic in $T$ and polynomial in $d$.
\end{ncbevfthm}

\begin{proof}
The move from $z_0$ to $z_1=u_1$ is at most $D$.  Consider a completed phase from $r_k$ to $r_{k+1}$.  The actual path follows the accepted chaser responses and then moves to $a_{k+1}$ instead of using the rejected response.  By the triangle inequality, its movement is no more than the chaser movement through the rejected response plus the distance from that response to $a_{k+1}$.  \Cref{lem:endpoint} bounds this by $\rho_d\delta_k$.

For the final unfinished phase, choose any point in $S_T$.  It belongs to every request in that phase and is at distance at most $D$ from the phase start.  The chaser therefore moves at most $\rho_dD$.  Hence
\[
        P_T
        \le
        D+\rho_d\sum_{k=0}^{m-1}\delta_k+\rho_dD.
\]
Use \Cref{lem:delta-sum} to obtain the displayed movement bound.  The regret and violation bounds are exactly \eqref{eq:coco-regret-movement}.
\end{proof}

Substituting $\rho_d=O(\sqrt{d\log(1+d)})$ into \Cref{thm:coco-main} gives
\[
\begin{aligned}
        \Reg_T
        &=O\!\left(
        \left[G_fD\sqrt{d\log(1+d)}
        +\frac{G_f^2}{\mu}d\log(1+d)\right]\log(eT)
        \right),\\
        \CCV_T
        &=O\!\left(
        \left[G_gD\sqrt{d\log(1+d)}
        +\frac{G_fG_g}{\mu}d\log(1+d)\right]\log(eT)
        \right).
\end{aligned}
\]

\subsection{General convex losses}
\label{sec:coco-convex}

When the losses are only convex, $F_t$ need not control the distance between two constrained minimizers.  We create this control by adding one fixed quadratic term.  Choose $\lambda>0$ and define
\[
        H_t(x):=\frac{\lambda}{2}\|x-z_0\|^2+\sum_{s=1}^t f_s(x),
        \qquad
        u_t:=\argmin_{x\in S_t}H_t(x),
        \qquad
        v_t:=H_t(u_t).
\]
At time zero, $H_0(x)=\frac{\lambda}{2}\|x-z_0\|^2$, so $u_0=z_0$ and $v_0=0$.  The algorithm uses the same loss test as before, now with the regularized minimum $v_t$.

\begin{algorithm}[t]
\caption{One-round-delayed chasing for convex \COCO}
\label{alg:coco-convex}
\begin{algorithmic}[1]
\Require Initial point $z_0\in X$, parameter $\lambda>0$, resettable chaser $\Chase$
\State Set $S_0\gets X$, $F_0\equiv0$, $A_0\gets0$, and start $\Chase$ at $z_0$.
\For{$t=1,2,\ldots,T$}
    \State Play $x_t\gets z_{t-1}$ and observe $f_t,g_t$.
    \State Set $S_t\gets S_{t-1}\cap\{x\in X:g_t(x)\le0\}$ and $F_t\gets F_{t-1}+f_t$.
    \State Define $H_t(x)\gets\frac{\lambda}{2}\|x-z_0\|^2+F_t(x)$.
    \State Compute $u_t\gets\argmin_{x\in S_t}H_t(x)$ and $v_t\gets H_t(u_t)$.
    \State Compute the tentative chaser response $y_t\in S_t$.
    \If{$A_{t-1}+f_t(y_t)\le v_t$}
        \State Set $z_t\gets y_t$ and retain the updated chaser history.
    \Else
        \State Set $z_t\gets u_t$, discard the tentative response, and restart $\Chase$ at $u_t$.
    \EndIf
    \State Set $A_t\gets A_{t-1}+f_t(z_t)$.
\EndFor
\end{algorithmic}
\end{algorithm}

\begin{ncbevflem}[Regularized loss bound]
\label{lem:convex-budget}
For every $t\ge0$, \Cref{alg:coco-convex} satisfies $A_t\le v_t$.
\end{ncbevflem}

\begin{proof}
At time zero, $A_0=v_0=0$.  Suppose $A_{t-1}\le v_{t-1}$.  Acceptance gives $A_t\le v_t$ directly.  After a rejection, $z_t=u_t$.  Since $u_t\in S_t\subseteq S_{t-1}$,
\[
        v_{t-1}\le H_{t-1}(u_t).
\]
Therefore
\[
\begin{aligned}
        A_t
        &=A_{t-1}+f_t(u_t)\\
        &\le v_{t-1}+f_t(u_t)\\
        &\le H_{t-1}(u_t)+f_t(u_t)
        =H_t(u_t)=v_t.
\end{aligned}
\]
\end{proof}

Let $0=r_0<r_1<\cdots<r_m\le T$ be the reset times of \Cref{alg:coco-convex}.  Set $a_0=z_0=u_0$, $a_k=z_{r_k}=u_{r_k}$ for $k\ge1$, and $\delta_k:=\|a_{k+1}-a_k\|$.

\begin{ncbevflem}[Reset distances for convex losses]
\label{lem:convex-reset}
For each completed phase, $\delta_k\le 2G_f\rho_d(r_{k+1}-r_k)/\lambda$.  Consequently, $\sum_{k=0}^{m-1}\delta_k\le 2G_f\rho_dT/\lambda$.
\end{ncbevflem}

\begin{proof}
Fix a completed phase and write $r=r_k$, $s=r_{k+1}$, $a=a_k$, $b=a_{k+1}=u_s$, and $\delta=\|a-b\|$.  Let $y_t$ denote the tentative chaser response for $t=r+1,\ldots,s$.  The responses through round $s-1$ are accepted, while the response at round $s$ is rejected.  Hence
\[
        A_{s-1}
        =A_r+\sum_{t=r+1}^{s-1}f_t(y_t)
\]
and the failed test at round $s$ gives
\[
        A_{s-1}+f_s(y_s)>v_s=H_s(b)
        =H_r(b)+\sum_{t=r+1}^{s}f_t(b).
\]
Subtracting the common terms yields
\begin{equation}
\label{eq:convex-reset-gap}
        \sum_{t=r+1}^s\bigl(f_t(y_t)-f_t(b)\bigr)
        >H_r(b)-A_r.
\end{equation}
Here $A_r\le v_r=H_r(a)$ by \Cref{lem:convex-budget}.  The function $H_r$ is $\lambda$-strongly convex, $a$ minimizes it over $S_r$, and $b\in S_s\subseteq S_r$.  Hence \Cref{lem:constrained-strong-convexity} gives
\[
        H_r(b)-A_r
        \ge H_r(b)-H_r(a)
        \ge \frac{\lambda}{2}\delta^2.
\]
By \Cref{lem:endpoint}, $\|y_t-b\|\le\rho_d\delta$.  Thus the left side of \eqref{eq:convex-reset-gap} is at most $(s-r)G_f\rho_d\delta$.  If $\delta>0$, division gives
\[
        \delta<\frac{2G_f\rho_d(s-r)}{\lambda};
\]
when $\delta=0$, the same bound is immediate.  Summing over the disjoint phase intervals proves the second claim.
\end{proof}

\begin{ncbevflem}[Movement of the regularized algorithm]
\label{lem:convex-movement}
The feasible sequence produced by \Cref{alg:coco-convex} satisfies $P_T\le \rho_dD+2G_f\rho_d^2T/\lambda$.
\end{ncbevflem}

\begin{proof}
As in the proof of \Cref{thm:coco-main}, the movement in a completed phase, including its reset, is at most $\rho_d\delta_k$.  In the final unfinished phase, append any point of $S_T$ in the analysis; the chaser movement is at most $\rho_dD$.  Therefore
\[
        P_T
        \le
        \rho_dD+\rho_d\sum_{k=0}^{m-1}\delta_k.
\]
Now apply \Cref{lem:convex-reset}.
\end{proof}

\begin{ncbevfthm}[Convex \COCO]
\label{thm:coco-convex-main}
Under \Cref{ass:coco-convex}, let \Cref{alg:coco-convex} use a resettable $\rho_d$-competitive chaser.  For every $\lambda>0$,
\[
        P_T\le \rho_dD+\frac{2G_f\rho_d^2T}{\lambda},
        \qquad
        \Reg_T\le\frac{\lambda D^2}{2}+G_fP_T,
        \qquad
        \CCV_T\le G_gP_T.
\]
If $D,G_f>0$, choosing $\lambda=2G_f\rho_d\sqrt T/D$ gives $\Reg_T=O(G_fD\rho_d\sqrt T)$ and $\CCV_T=O(G_gD\rho_d\sqrt T)$.  With \Cref{thm:known-chaser}, the dependence on the dimension is $O(\sqrt{d\log(1+d)})$.  Alternatively, choosing $\lambda=2G_f\rho_d^2\sqrt T/D$ gives $\Reg_T=O(G_fD\rho_d^2\sqrt T)$ and $\CCV_T=O\bigl(G_gD(\sqrt T+\rho_d)\bigr)$.
\end{ncbevfthm}

\begin{proof}
The movement bound is \Cref{lem:convex-movement}.  Let $x^\star$ minimize $\sum_{t=1}^T f_t(x)$ over $S_T$.  By \Cref{lem:convex-budget},
\[
        A_T\le v_T\le H_T(x^\star).
\]
Hence
\[
        A_T-\sum_{t=1}^T f_t(x^\star)
        \le
        \frac{\lambda}{2}\|x^\star-z_0\|^2
        \le
        \frac{\lambda D^2}{2}.
\]
The two inequalities in \Cref{lem:coco-delay} now give the displayed regret and violation bounds.  Substituting the two values of $\lambda$ gives the two stated pairs of rates.
\end{proof}

The parameter-free inequalities in \Cref{thm:coco-convex-main} hold for every $\lambda>0$.  The displayed choices use the horizon only to balance their two terms.  If $D=0$, the decision set is a singleton.  If $G_f=0$, the losses are constant on $X$ and any positive $\lambda$ may be used.

\section{Conclusion}
\label{sec:conclusion}

The same simple rule underlies all the algorithms in this paper.  A nested-body chaser proposes a feasible point with controlled movement.  The algorithm uses the proposal when it preserves a cumulative-loss bound and otherwise moves to a constrained minimizer.  This separates the geometric problem of following shrinking sets from the loss calculation that decides when movement is necessary.

For strongly convex \CONES, restricting the chaser to carefully chosen objective sublevel sets gives nonpositive regret and $O(\rho_d\sqrt{\log T})$ movement.  The two-dimensional lower bound shows that the $\sqrt{\log T}$ dependence cannot be improved under any fixed polynomial sublinear terminal-regret guarantee.  When the objective grows linearly away from each constrained minimizer set, the behavior changes: Steiner-point tracking has movement independent of the horizon.

For \COCO, the learner computes a feasible point after each round and plays it one round later.  Regret and constraint violation are then controlled by the movement of this feasible sequence.  Strong convexity gives logarithmic bounds, while a fixed quadratic term gives $O(\sqrt T)$ bounds for general convex losses.  In each case, using Euclidean nested-body chasing yields explicit polynomial dependence on the dimension.

The algorithms use exact constrained minimization and a deterministic nested-body chaser; the Steiner-point result also requires exact minimizer sets and Steiner points.  An important next question is whether the same dimension dependence can be obtained using only projections, separation queries, or first-order information.  It is also open to determine the best joint dependence on dimension, regret, and cumulative constraint violation.

\section*{Acknowledgments}

The authors are deeply grateful to Professor Abhishek Sinha (TIFR, Mumbai), who first opened the door to this beautiful area of theoretical machine learning when they were still starry-eyed undergraduates at IIT Kharagpur. His generosity, encouragement, and way of seeing the subject have remained with them ever since; this paper owes much to the spark he kindled in those early years. They are also grateful to Professor Rahul Vaze (TIFR, Mumbai), whose thoughtful feedback helped improve the presentation and positioning of the paper.

\clearpage
\bibliographystyle{plainnat}
\bibliography{refs}

\clearpage
\appendix
\section{Comparison with the two versions of the original CONES paper}
\label{app:versions}

The paper of \citet{krishna2026conesv1,krishna2026conesv2} changed substantially between its first and second arXiv versions.  This appendix separates the two versions and explains the comparison made in \Cref{sec:related}.  We focus on the strongly convex and uniformly sharp results, which are the parts directly related to this paper.

\paragraph{Version 1: the strongly convex lower-bound parameters.} Version~1 proposed an $O(\log T)$ movement upper bound for the Frugal algorithm and an $\Omega(\log T)$ lower bound for every algorithm with sublinear regret.  In the strongly convex lower-bound construction, its parameters are required to satisfy
\begin{equation}
\label{eq:v1-conditions}
        a>\frac{B^2}{4\delta},
        \qquad
        \delta\le\frac{b-a}{K}.
\end{equation}
The first inequality is equivalent to
\[
        \delta>\frac{B^2}{4a}.
\]
Combining it with the second condition in \eqref{eq:v1-conditions} gives
\[
        \frac{B^2}{4a}<\frac{b-a}{K},
        \qquad\text{and hence}\qquad
        K<\frac{4a(b-a)}{B^2}.
\]
When $a$, $b$, $B$, and the domain are fixed independently of $T$, the number of phases $K$ is therefore bounded by a constant.  Under those parameter requirements, the construction cannot produce a lower bound from a number of phases proportional to $\log T$.

\paragraph{Version 1: the proposed uniformly sharp example.}

Version~1 also proposed a logarithmic lower bound for an objective of the form
\[
        f(x)=c\|x\|.
\]
Its constrained minimizers are nonzero points.  Let $u\ne0$ be one of these minimizers and take a feasible displacement $h$ tangent to the active supporting line, so that $h\perp u$.  Then
\[
\begin{aligned}
        \frac{f(u+h)-f(u)}{\|h\|}
        &=c\frac{\|u+h\|-\|u\|}{\|h\|}\\
        &=c\frac{\|h\|}{\|u+h\|+\|u\|}
        \longrightarrow0
        \qquad\text{as }h\to0.
\end{aligned}
\]
Thus no positive constant $\alpha$ can make
\[
        f(x)-f(u)\ge\alpha\|x-u\|
\]
hold along these feasible tangent directions.  The proposed example does not satisfy the uniform set-relative sharpness condition used for that lower bound.

This point is consistent with \Cref{thm:cones-sharp}.  Under uniform linear growth away from the current minimizer set, the Steiner-point algorithm has a movement bound that does not depend on $T$.  A lower bound that grows with $T$ cannot hold under the same assumption.

\paragraph{Version 2: the repaired strongly convex construction.}

Version~2 changes the scaling in the strongly convex construction to
\[
        B=K^{-1/2},
        \qquad
        \delta=\frac{b-a}{K}.
\]
Now
\[
        \frac{B^2}{4\delta}=\frac{1}{4(b-a)},
\]
which is independent of $K$ and removes the obstruction above.  The revised proof also tracks a phase-growth factor $\gamma=\Theta(K)$.  As a result, the horizon grows roughly as $\gamma^K$ rather than with a constant base, and the admissible number of phases is
\[
        K=\Theta\!\left(\frac{\log T}{\log\log T}\right).
\]
Since the movement forced by the construction is of order $\sqrt K$, the formal lower-bound theorem in Version~2 is
\[
        \Omega\!\left(\sqrt{\frac{\log T}{\log\log T}}\right).
\]
The theorem assumes a regret bound at every prefix.  The formal uniformly sharp lower-bound theorem from Version~1 is not present in Version~2.

A few statements in Version~2 still use the earlier $\Omega(\log T)$ claim.  They occur in the opening sentence of Section~4.2, in the paragraph following Theorem~9 in Section~5, and in the final comparison paragraphs of Sections~10.1, 10.2, and 10.3.  The transition into Section~4.2 also calls the lower bound matching, and the conclusion describes the $O(\log T)$ upper bound as minimax optimal.  These statements do not agree with the formal theorem and its final calculation, which give the square-root rate above.  Our comparison uses the formal Version~2 theorem.

\paragraph{Comparison with the present paper.}

\begin{table}[t]
\centering
\caption{Strongly convex \CONES: Version~2 and the present paper.}
\label{tab:cones-v2-comparison}
\renewcommand{\arraystretch}{1.25}
\begin{tabularx}{\textwidth}{>{\raggedright\arraybackslash}p{0.25\textwidth}>{\raggedright\arraybackslash}X>{\raggedright\arraybackslash}X}
\toprule
 & \textbf{Krishna et al., Version 2} & \textbf{This paper}\\
\midrule
Upper-bound dependence on $T$
& $O(\log T)$ movement
& $O(\rho_d\sqrt{\log T})$ movement\\
Objective assumptions
& smooth and strongly convex
& Lipschitz and strongly convex; no smoothness assumption\\
Dimension factor
& projection-path factor $O(d^{d/2})$
& $\rho_d=O(\sqrt{d\log(1+d)})$\\
Formal lower-bound rate
& $\Omega(\sqrt{\log T/\log\log T})$
& $\Omega(\sqrt{\log(T/(R_T+1))})$\\
Regret condition in the lower bound
& bound imposed at every prefix
& terminal expected regret at most $R_T$\\
Randomized algorithms
& not part of the formal strengthening
& included explicitly; the hard sequence is deterministic after the algorithm is fixed\\
\bottomrule
\end{tabularx}
\end{table}

The upper-bound algorithm combines two design principles already present in the original \CONES paper.  Its outer cumulative-loss test and forced return to a constrained minimizer are inherited conceptually from \textsc{Frugal}.  Its localized requests $S_t\cap\{f\le\text{level}\}$ follow LSP, whose Steiner-point variant already replaces successive projections by a more global selector.  The present algorithm does more than compose these two rules: it runs a start-aware, stateful nested-body chaser on adaptive geometrically increasing value scales, previews and conditionally discards the chaser state, and uses endpoint augmentation to include the final reset jump in the same competitive bound.  Strong convexity then converts a phase value increase $\Delta$ into movement $O(\rho_d\sqrt{\Delta/\mu})$.  Summing these square-root charges is what changes the dependence on the horizon from $\log T$ to $\sqrt{\log T}$, removes the smoothness assumption, and yields the polynomial factor $\rho_d$.

The lower bounds share essentially the same planar geometric construction and forcing principle.  Both use the quadratic objective, nested tangent halfspaces, alternating constrained minimizers separated by $\Theta(K^{-1/2})$, and a rule that withholds the next cut until the learner approaches the current minimizer.  The main differences are in the admissible regret guarantee and the accounting across phases.  Version~2 assumes a regret bound at every prefix and uses a phase-growth factor $\Theta(K)$, leading to $K=\Theta(\log T/\log\log T)$.  In \Cref{thm:terminal-strongly-convex-cones-lower}, a stop-or-continue rule makes the construction valid under only terminal regret: failure to reach the current target causes the set sequence to freeze, making the current optimum the final benchmark.  The proof also keeps the exact past benchmark credit and uses the local increment $v_{k+1}-v_k=O(K^{-1})$, which gives a constant-factor recurrence and permits $K=\Theta(\log(T/(R_T+1)))$.  Both proofs allow negative regret accumulated in earlier phases; the improvement does not follow merely from replacing a prefixwise condition by a terminal one.  Mean-action derandomization further extends the present result to randomized algorithms.

Version~2 was posted on 16 August 2026.  The sublevel-chasing upper bound and the terminal-regret lower-bound proof in this paper were developed independently and concurrently with that revision.  The comparison above acknowledges the common FRUGAL/LSP design principles and the common lower-bound geometry while isolating the additional algorithmic and analytic ingredients needed for the results proved here.

\section{Nested-body chasing for lower-dimensional requests}
\label{app:compact-requests}

The Euclidean chasing theorem used in \Cref{thm:known-chaser} is usually stated for convex bodies with nonempty interior.  Our algorithms also use intersections that may lie in a lower-dimensional affine subspace, and the proof of \Cref{lem:endpoint} appends a singleton.  We now show that the same competitive ratio extends to every nonempty compact convex request.

A chaser is \emph{prefix-consistent} if extending a request sequence does not change the outputs already produced on its earlier requests.  Every deterministic online chaser has this property.

For a starting point $a$ and requests $K_1\supseteq\cdots\supseteq K_j$, let
\[
        \operatorname{OPT}(a;K_1,\ldots,K_j)
        :=
        \inf_{\substack{w_i\in K_i\\w_0=a}}
        \sum_{i=1}^j\|w_i-w_{i-1}\|.
\]

\begin{ncbevfprop}[Extension to compact convex requests]
\label{prop:compact-extension}
Suppose a deterministic, prefix-consistent, $\rho_d$-competitive chaser is available for full-dimensional compact convex requests and can be started from any point.  Then there is a deterministic resettable chaser with the same competitive ratio for arbitrary nonempty compact convex requests.
\end{ncbevfprop}

\begin{proof}
Fix a starting point $a$ and a finite nested request prefix
\[
        K_1\supseteq\cdots\supseteq K_j.
\]
For each integer $n\ge1$, enlarge every request slightly:
\[
        K_i^{(n)}:=K_i+n^{-1}\B,
        \qquad i=1,\ldots,j.
\]
These sets are full-dimensional, compact, convex, and nested.  Run the original chaser on this enlarged prefix and denote its output vector by
\[
        Y_j^{(n)}=(y_1^{(n)},\ldots,y_j^{(n)}).
\]
For all large $n$, every coordinate lies in the compact set $K_1+\B$.  Hence the sequence of output vectors has cluster points.  Let $C_j$ be the compact set of all such cluster points.  Using a fixed coordinate basis and ordering the coordinates by time, choose the lexicographically first point $\overline Y_j$ in $C_j$.

These choices agree across prefixes.  Indeed, the projection of $C_j$ onto its first $j-1$ output blocks is exactly $C_{j-1}$.  One inclusion follows because the original chaser is online: the first $j-1$ outputs do not depend on the last request.  For the reverse inclusion, take a subsequence whose first $j-1$ blocks converge to any chosen point of $C_{j-1}$.  Compactness provides a further subsequence on which the last block also converges, producing a point of $C_j$ with the desired prefix.  Since the coordinates are ordered by time, the lexicographically first point of $C_j$ must project to the lexicographically first point of $C_{j-1}$.  Therefore the last block of $\overline Y_j$ defines a deterministic online response to the prefix $K_1,\ldots,K_j$.

Each response belongs to its original request.  Along any convergent subsequence,
\[
        y_i^{(n)}\in K_i+n^{-1}\B
        \quad\Longrightarrow\quad
        \dist(y_i^{(n)},K_i)\le n^{-1}.
\]
The limit lies in $K_i$ because $K_i$ is closed.

It remains to prove the movement bound.  For every $n$,
\[
        \sum_{i=1}^j\|y_i^{(n)}-y_{i-1}^{(n)}\|
        \le
        \rho_d\operatorname{OPT}(a;K_1^{(n)},\ldots,K_j^{(n)}),
        \qquad y_0^{(n)}=a.
\]
Since $K_i\subseteq K_i^{(n)}$,
\[
        \operatorname{OPT}(a;K_1^{(n)},\ldots,K_j^{(n)})
        \le
        \operatorname{OPT}(a;K_1,\ldots,K_j).
\]
Take a subsequence converging to $\overline Y_j$ and use continuity of the norm.  The limiting outputs satisfy the same $\rho_d$-competitive inequality.  Repeating the construction from any starting point makes the new chaser resettable.
\end{proof}

Applying \Cref{prop:compact-extension} to the full-dimensional Euclidean chaser of \citet{bubeck2020nested} proves the compact-request statement in \Cref{thm:known-chaser}.

\end{document}